\documentclass{article}
\usepackage{spconf,amsmath,graphicx,hyperref}
\usepackage{booktabs}%提供命令\toprule、\midrule、\bottomrule
\usepackage{multirow} % for cmd 'multirow', 'multicolumn'
\usepackage{amsmath}
\newtheorem{theorem}{Theorem}

\newtheorem{proof}{Proof}[section]
\usepackage{makecell}
\usepackage{enumitem}
\usepackage{amsfonts}
\usepackage{etoolbox}
\makeatletter
\patchcmd{\thebibliography}
  {\list} % 找到正文 list 开始的地方
  {\ninept\list} % 在正文开始之前加 \small
  {}{}
\makeatother

\title{AR-LIF: Adaptive Reset Leaky Integrate-and-Fire Neuron for Spiking Neural Networks}
\name{Zeyu Huang, Wei Meng, Quan Liu, Kun Chen and Li Ma*\thanks{* Corresponding author. excellentmary@whut.edu.cn}
\thanks{This work was supported by National Natural Science Foundation
of China (62572364), Natural Science Foundation of Wuhan Municipality (2024040801020263), Science and Technology Talents Serving Enterprise Program of Hubei Province (2025DJB030). 
}}

\address{School of Information Engineering, Wuhan University of Technology, Wuhan, China}
\begin{document}
\topmargin=0mm
% \ninept
%
\maketitle
\begin{abstract}
Spiking neural networks offer low energy consumption due to their event-driven nature. Beyond binary spike outputs, their intrinsic floating-point dynamics merit greater attention. Neuronal threshold levels and reset modes critically determine spike count and timing. Hard reset cause information loss, while soft reset apply uniform treatment to neurons. To address these issues, we design an adaptive reset neuron that establishes relationships between inputs, outputs, and reset, while integrating a simple yet effective threshold adjustment strategy. Experimental results demonstrate that our method achieves excellent performance while maintaining lower energy consumption. In particular, it attains state-of-the-art accuracy on Tiny-ImageNet and CIFAR10-DVS. Codes are available at \href{https://github.com/2ephyrus/AR-LIF}{https://github.com/2ephyrus/AR-LIF}.
\end{abstract}
\section{Introduction}
\label{sec:intro}

As the third generation of neural networks \cite{maass1997networks}, Spiking Neural Networks (SNNs) possess an event-driven nature. They transmit information using 0/1 spikes and enjoy the energy efficiency advantage of multiplication-free inference by converting multiplications into additions \cite{meng2025sd}. A current research focus is how to achieve performance comparable to or even surpassing that of Convolutional Neural Networks (CNNs) while preserving the event-driven property \cite{yao2023sdt, yao2025scalingv3, stssa}. Discrete and iterable spiking neuron models \cite{wu2018stbp} 
 have facilitated the development of algorithms for direct training of SNNs \cite{deng2022tet, duan2022tebn}. The oversimplification of spiking neuron dynamics \cite{liu2023reconstruction} and the neglect of spiking neuron heterogeneity \cite{zhang2025dalif} are key factors contributing to the suboptimal performance of SNNs. For spiking neurons, different reset modes and threshold levels significantly affect the frequency and number of spikes fired \cite{ai2025cross}. Neurons employing "hard reset" reset diverse membrane potentials to zero, which reduces the information encoding capability of the network, causes information loss, and leads to performance degradation \cite{guo2022reducing}. Neurons using "soft reset" have a post-reset potential greater than zero, risking over-activation; moreover, they ignore neuronal heterogeneity, limiting the expressive power of spiking neurons \cite{zhang2024tc}.

To address these issues, we propose an adaptive reset leaky integrate-and-fire (AR-LIF) neuron. Based on observations of the simplified soft-reset spiking neuron equation, we find that the output of a spiking neuron is closely related to historical inputs and outputs. Thus, regulating the temporal dynamics of spiking neurons using historical inputs and outputs becomes a reasonable approach. We introduce an additional memory variable into the neuron dynamics, which is responsible for computing nonlinear operations on inputs and outputs. This variable, featuring temporal decay properties, acts as a feedback signal to participate in the neuron's reset process. Additionally, existing reset modes are closely linked to threshold voltage. To enhance the heterogeneity of spiking neurons \cite{wang2022ltmd}, we express the threshold voltage as a nonlinear function of the input at the same moment, meaning the threshold level is spatiotemporally independent.

Fig. \ref{fig_1} (b) visualizes the neuron outputs, showing that AR-LIF has a lower spike firing rate.
Fig. \ref{fig_1} (c) presents the membrane potential distribution of neurons under inputs with perturbed normal distribution. AR-LIF exhibits lower JS divergence and Wasserstein Distance, and according to \cite{ding2025rethinking}, the greater the difference in the membrane potential distribution of neurons across time steps, the poorer the network performance, it can be inferred that AR-LIF has better performance.

Our proposed method has been validated on multiple static and neuromorphic datasets, achieving state-of-the-art (SOTA) results. Extensive ablation experiments demonstrate the robustness and effectiveness of the proposed model.

% AR-LIF mainly makes two adjustments: first, it adds an extra reset voltage on the basis of soft reset; second, it associates the neuron threshold with the input, which means that the threshold of AR-LIF is independent layer by layer and time step by time step.

\begin{figure*}
\centering
\includegraphics[width=6.45 in, keepaspectratio]{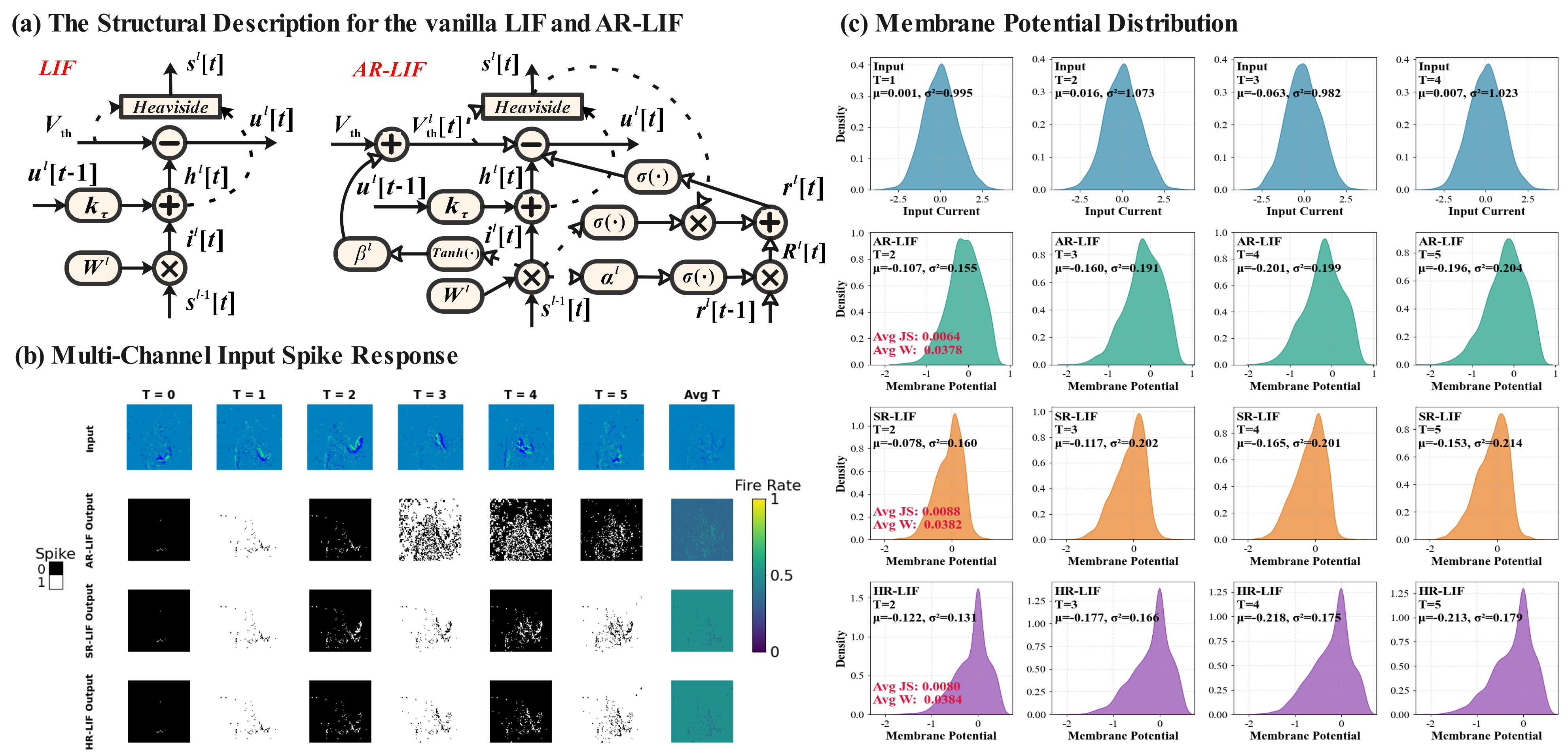}\\
\caption{(a) The hollow arrows in the figure indicate the additional computational processes introduced by AR-LIF. The dashed lines represent the replication of the solid-line outputs of the nodes. (b) Spike responses to 3-channel gesture video inputs. (c) The membrane potential distribution of neurons under inputs with perturbed normal distributions.
}
\label{fig_1}
\end{figure*}

\section{Methods}
\label{sec:format}
\subsection{Leaky Integrate-and-Fire (LIF) Neuron}
The discrete and iterative LIF neuron's dynamical equations can be described by three processes:

\textbf{\textit{Leaky and integrate:}}
The membrane potential of neurons is composed of the self-decayed membrane potential plus the input current.
\begin{equation}
h^l[t]=k_\tau u^l[t-1]+i^l[t],
\end{equation}
where $i^l[t]=W^l * s^{l-1}[t]$, $W^l$ is the synaptic weight, $t$ is time step, and $l$ denotes the $l$-$th$ layer of the network. $k_\tau \in(0,1)$ is the neuron membrane decay factor, which is a constant. $i^l[t]$ is the input of the neuron, $u^l[t-1]$ is the membrane potential after spike firing, $h^l[t]$ is the integrated membrane potential, $s^{l-1}[t]$ is the spike output of the neuron.

\textbf{\textit{Fire:}}
When the membrane potential after leaky integration satisfies $h^l[t]>V_{\mathrm{th}}$ , the output is 1; otherwise, the output is 0.
\begin{equation}
s^l[t]=\Theta\left(h^l[t]-V_{\text {th}}\right),
\end{equation}
where $\Theta(\cdot)$ is the Heaviside function, $V_{\mathrm{th}}$ is the neuronal threshold.

\textbf{\textit{Reset:}}
After a spiking neuron fires a spike, it needs to reset to restore the voltage to a steady state.
\begin{equation}
u^l[t]=\left\{\begin{array}{c}
h^l[t] \odot\left(1-s^l[t]\right), \text { Hard Reset } \\
h^l[t]-\rho s^l[t], \text { Soft Reset }
\end{array}\right.
\end{equation}
where $ \odot\ $ denotes element-wise matrix multiplication. In general, $\rho = V_{\text {th}}$, in some studies \cite{zhang2024tc}, it is a learnable parameter.

\subsection{Motivation of Adaptive Reset Neuron Design}
\begin{theorem}
\label{the}
For a spiking neuron using soft reset, the spike it fires at time $t$ can be expressed as an iterative formula of input and output, independent of the membrane potential.
\end{theorem}
\begin{proof}
Let the threshold of the neuron be $V_{\text {th}}$, and $u^l[t]$ denotes the membrane potential of the neuron after the spike firing process, and $h^l[t]$ denotes the membrane potential after leaky integration. We have $h^l[t]=\tau u^l[t-1]+i^l[t]$, $u^l[t]=h^l[t]-s^l[t] \cdot V_{\text {th}}$. It can be derived that: 
\begin{equation}
\begin{aligned}
& s^l[t]=\Theta\left(h^l[t]-V_{\text {th}}\right) \\
& =\Theta\left(\left(\tau\left(h^l[t-1]-V_{\text {th }} s^l[t-1]\right)\right)+i^l[t]-V_{\text {th}}\right) \\
& =\Theta(\underbrace{i^l[t]+\tau i^l[t-1]+\cdots+\tau^{t-1} i^l[1]}_{\text {Iterative Formula of Input }} \\
& -V_{\text {th}}(1+\underbrace{\left.\tau s^l[t-1]+\cdots+\tau^{t-1} s^l[1]\right)}_{\text {Iterative Formula of Output }})
\end{aligned}
\end{equation}
\end{proof}

\textbf{Observation 1: \textit{Limitations of existing reset modes.}}
As prior work notes, hard reset’s homogeneous mechanism causes significant information loss \cite{guo2022reducing}. Soft reset avoids resetting spiking neurons’ potential to 0, instead uniformly subtracting the threshold voltage from them, thus retaining voltage exceeding the threshold but risking excessive activation. Though a learnable parameter $\rho$ enables adaptive adjustment during training, $\rho$ is constrained to positive values \cite{zhang2024tc}, leaving the over-activation risk unresolved. Moreover, $\rho$ is shared across neurons in the same layer, lacking neuron-level heterogeneity.

\textbf{Observation 2: \textit{Spiking neurons' output using soft reset is an iteration of inputs and outputs.}}
By Theorem \ref{the}, a neuron's spiking at time $t$ is closely linked to cumulative attenuated inputs and spikes from time step 0 to $t$. Since reset voltage indirectly affects spiking by influencing membrane potential, leveraging cumulative inputs and outputs to construct reset voltage constitutes a viable attempt.

\textbf{Observation 3: \textit{Inspiration from neuronal feedback regulation mechanisms.}}
Current studies have shown that the excitation and inhibition \cite{zhao2022backeisnn} of biological neurons are closely related to inputs \cite{zhang2024tc}. Therefore, regulating neuronal spike firing using input information is biologically plausible \cite{zheng2024den}. 

Based on the above observations on existing neurons, we attempt to design a novel spiking neuron model that integrates an adaptive reset mechanism with a threshold adjustment strategy.

\subsection{Adaptive Reset Leaky Integrate-and-Fire Neuron}

\textbf{\textit{Adaptive decay of input accumulations:}}
Consider introducing an additional memory variable $r^l[t]$ into the LIF neuron. $r^l[t]$  has an decay property, and the strength of its decay is associated with the input. For negative values of $r^l[t]$, a complementary decay coefficient is adopted. It can be described as:
\begin{equation}
\begin{aligned}
& R^l[t] = \begin{cases} 
\sigma(\alpha^l i^l[t]) \cdot r^l[t-1], & r^l[t-1] \geq 0 ,\\ 
(1-\sigma(\alpha^l i^l[t])) \cdot r^l[t-1], & r^l[t-1]<0 ,
\end{cases}
\end{aligned}
\end{equation}
where $R^l[t]$ denotes an intermediate quantity prior to the accumulation that yields $r^l[t]$, $\alpha^l \in \mathbb{R}$ is a learnable parameter, $\sigma(\cdot)$ is the sigmoid function, $i^l[t]$ is the input at time step $t$, and $r^l[-1]=0$.

\textbf{\textit{Spike feedback calculation of input accumulations:}}
$r^l[t]$ is used to count the input penalty amount related to spikes and the input incentive amount without spikes. For inputs that generate spikes, they are accumulated normally; for inputs that do not generate spikes, reverse accumulation is performed. It can be described as:
\begin{equation}
\begin{aligned}
& r^l[t]=R^l[t]+s^l[t] \cdot \sigma(i^l[t])-(1-s^l[t]) \cdot \sigma(i^l[t]) ,\\
& = \begin{cases}R^l[t]+\sigma(i^l[t]), & s^l[t]=1, \\
R^l[t]-\sigma(i^l[t]), & s^l[t]=0,\end{cases}
\end{aligned}
\end{equation}
where $s^l[t]$ is the output of the spiking neuron at time step $t$.

\textbf{\textit{Adaptive reset voltage calculation:}}
To ensure the stability of the reset process, the adaptive reset voltage needs to be adjusted by the input accumulation with the threshold as the reference, which is specifically described as:
\begin{equation}
v_{\mathrm{r}}^l[t]=V_{\mathrm{th}}^l[t]+\sigma(r^l[t]),
\end{equation}
where $V_{\mathrm{th}}$ is the neuron threshold, which is usually set to 1. We enable the threshold to undergo small-range adaptive adjustment; specifically, $V_{\text {th}}^l[t]=1+\beta^l  \operatorname{Tanh}(i^l[t])$, $\operatorname{Tanh}(\cdot)$ is the hyperbolic tangent function, $\beta^l \in(-1,1)$. 

The adaptive reset process can be calculated as:
\begin{equation}
u^l[t]=h^l[t]-v_{\mathrm{r}}^l[t] \cdot s^l[t],
\end{equation}
where $v_{\mathrm{r}}^l[t]$ is the adaptive reset voltage.

\textbf{\textit{AR-LIF neuronal dynamics:}}
The structural description for AR-LIF can be found in Fig. \ref{fig_1} (a). Based on the above derivation, the dynamical equations of AR-LIF can be described as follows:
\begin{equation}
\label{eq9}
\left\{\begin{array}{c}
h^l[t]=k_\tau u^l[t-1]+i^l[t], \\
s^l[t]=\Theta\left(h^l[t]-V_{\text {th}}^l[t]\right), \\
R^l[t]=\left\{\begin{array}{c}
\sigma\left(\alpha^l  i^l[t]\right) \cdot r^l[t-1], r^l[t-1] \geq 0, \\
\left(1-\sigma\left(\alpha^l  i^l[t]\right)\right) \cdot r^l[t-1], r^l[t-1]<0,
\end{array}\right. \\
r^l[t]=R^l[t]+(2s^l[t]-1) \cdot \sigma(i^l[t]),\\
v_{\mathrm{r}}^l[t]=V_{\mathrm{th}}^l[t]+\sigma(r^l[t]),\\
u^l[t]=h^l[t]-v_{\mathrm{r}}^l[t] \cdot s^l[t],
\end{array}\right.
\end{equation}
where $r^l[t]$ is the self-feedback accumulation of input current.

\section{Experiments}
\label{sec:Experiments}
\subsection{Experimental Details}
We use the direct training method to train the network, and the loss function is defined based on TET \cite{deng2022tet}.
\begin{equation}
\mathcal{L}=\lambda \mathcal{L}_{\mathrm{TET}}+(1-\lambda) \mathcal{L}_{\mathrm{MSE}},
\end{equation}
where $\lambda$ is a constant used to control the proportion of each part of the loss function. $\mathcal{L}_{\mathrm{TET}}=\frac{1}{T} \sum_{t=1}^T \mathcal{L}_{\mathrm{CE}}(O(t), y)$, $\mathcal{L}_{\mathrm{MSE}}=\frac{1}{T} \sum_{t=1}^T \operatorname{MSE}(O(t), \phi)$. $O(t)$ is the output of the last layer of the network, and $y$ is the real label value. When $\mathcal{L}_{\text {MSE }}$ is used as a regular term, $\phi$ is a constant; otherwise, it represents the one-hot encoding of the real label value.

We use a rectangular function as a surrogate function to solve the problem of non-differentiability of spike activities in the direct training process:
\begin{equation}
\frac{\partial \mathcal{L}}{\partial W^l}=\sum_{t=0}^T\left(\frac{\partial \mathcal{L}}{\partial s^l[t]} \cdot \frac{\partial s^l[t]}{\partial h^l[t]}+\frac{\partial \mathcal{L}}{\partial u^l[t]} \cdot \frac{\partial u^l[t]}{\partial h^l[t]}\right) \frac{\partial h^l[t]}{\partial W^l}.
\end{equation}
\begin{equation}
\frac{\partial s^l[t]}{\partial h^l[t]}=\frac{1}{a} \cdot \operatorname{sign}\left(\left|h^l[t]-V_{\mathrm{th}}\right|<\frac{a}{2}\right),
\end{equation}
where $a$ is set to 1. When $V_{\text {th}}-0.5 \leq h^l[t] \leq V_{\text {th}}+0.5$, $\frac{\partial s^l[t]}{\partial h^l[t]}=1$ ; otherwise, $\frac{\partial s^l[t]}{\partial h^l[t]}=0$.

\begin{figure}
\centering
\includegraphics[width=3in, keepaspectratio]{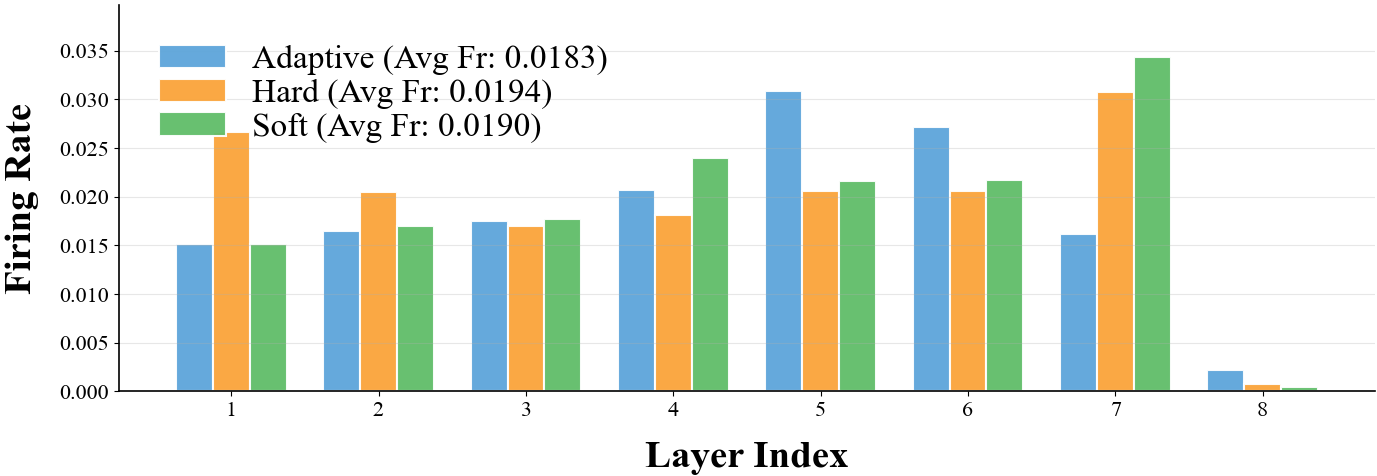}\\
\caption{Firing rates of different neurons across layers.
}
\label{fig_2}
\end{figure}

\begin{table}
\caption{Compare with previous studies on neuromorphic and static datasets.\label{tab2}}
\centering
\resizebox{0.48\textwidth}{!}{
\begin{tabular}{ccccc}
\toprule
Dataset & Method & Architecture & T & Acc.(\%)\\
\midrule
\multirow{4}{*}{CIFAR-10}  
 & TET\cite{deng2022tet} & ResNet-19 & 2/4/6 & 94.16/94.44/94.50\\
 & GLIF\cite{guo2022reducing} & ResNet-18 & 2/4/6 & 94.15/94.67/94.88\\
 & MPD-AGL\cite{jiang2025agl} & ResNet-19 & 2/4/6 & 96.18/96.35/96.54\\
  \cline{2-5}
 & Ours & ResNet-18 & 4/6/8 & 96.1/96.3/96.7\\
\midrule
\multirow{5}{*}{CIFAR-100} & GLIF\cite{yao2022glif}	& ResNet-18	& 6 & 77.28\\
 & IM-LIF\cite{lian2024imlif} & ResNet-19 & 3/6 & 77.21/77.42\\
 & DA-LIF\cite{zhang2025dalif} & ResNet-19 & 1/2/4 & 79.77/80.59/80.16\\
 & MPD-AGL\cite{jiang2025agl} & ResNet-19 & 2/4/6 & 78.84/79.72/80.49\\
  \cline{2-5}
 & Ours & ResNet-18 & 4/6/8 & 80.1/80.8/81.0\\
\midrule
\multirow{6}{*}{Tiny-ImageNet} & IM-LIF\cite{lian2024imlif} & ResNet-19 & 6 & 55.37\\
 & MPD-AGL\cite{jiang2025agl} & VGG-13 & 4 & 58.14\\
& LM-H\cite{hao2023lmh} & VGG-13 & 4 & 59.93\\
 & CLIF\cite{huang2024clif} & VGG-13 & 4/6 & 63.16/64.13\\
  \cline{2-5}
 & \multirow{2}{*}{Ours} & VGG-13 & 4/6 & 66.80/67.51\\
 &  & VGGSNN & 4/6 & 68.75/68.93\\
\midrule
\multirow{5}{*}{CIFAR10-DVS} & TEBN\cite{duan2022tebn} & VGGSNN & 10 & 84.9\\
 & DA-LIF\cite{zhang2025dalif} & VGG-16 & 16 & 82.42\\
 & AGMM\cite{agmm} & VGGSNN & 10 & 82.40\\
 & FPT\cite{feng2025fpt} & VGG-11 & 10 & 85.7\\
  \cline{2-5}
 & Ours & VGGSNN & 4/8/16 & 86.9/87.6/87.9\\
\midrule
\multirow{4}{*}{DVS-Gesture}  & MPE-PSN\cite{chen2025mpe} & 7B-Net & 16/20 & 97.92/97.22\\
 & DA-LIF\cite{zhang2025dalif} & VGG-11	& 16	& 98.61\\
 & FPT\cite{feng2025fpt} & VGG-11 & 20 & 98.61\\
  \cline{2-5}
 & Ours & VGG-11 & 20 & 98.61\\
\bottomrule
\end{tabular}
}
\end{table}

\subsection{Comparisons with Existing Methods}
We evaluated the performance of AR-LIF on multiple static and neuromorphic datasets, and the results of their comparison with SOTA methods are presented in Table \ref{tab2}.

\textbf{CIFAR-10 \& CIFAR-100}.
Compared with GLIF \cite{yao2022glif}, our method remains leading on the ResNet-18 architecture. AR-LIF yields accuracies of 96.7\% and 81.0\% at time step 8, performing on par with the leading results in SNNs.

\textbf{Tiny-ImageNet}.
In comparison with CLIF \cite{huang2024clif}, our method maintains a leading advantage of more than 3\% on the VGG-13 architecture. Notably, our method achieves accuracies of 68.75\% (T=4) and 68.93\% (T=6) on the VGGSNN architecture, representing the SOTA performance in SNNs.

\textbf{CIFAR10-DVS} \cite{li2017dvs}.
Our method achieves accuracies of 86.9\%, 87.6\%, and 87.9\% at 4, 8, and 16 time steps, respectively. This represents the SOTA result on CIFAR10-DVS.

\textbf{DVS-Gesture} \cite{amir2017ges}.
On the VGG-11, our method can also achieve an accuracy of 98.61\% at 20 time steps. This demonstrates the broad effectiveness of our method.

\begin{table}
\caption{Ablation study. "Both" denotes the scenario where $\alpha^l=1$ and $V_{\text {th}}^l[t]=1$, while "Learnable" denotes AR-LIF.\label{tab3}}
\centering
\resizebox{0.45\textwidth}{!}{
\begin{tabular}{ccccccc}
\toprule
\multirow{3}{*}{\makecell{Dataset \\ 4 Time Steps}} & \multicolumn{6}{c}{Accuracy(\%)} \\ 
\cline{2-7}
& \multirow{2}{*}{Hard} & \multirow{2}{*}{Soft} & \multicolumn{4}{c}{Adaptive} \\ 
\cline{4-7}
& & & $\alpha^l=1$ & $V_{\text {th}}^l[t]=1$ & Both & Learnable \\
\midrule
CIFAR100 & 73.4 & 63.4 & 77.0 & 76.5 & 76.1 & 77.8\\
\midrule
CIFAR10DVS & 83.8 & 84.1 & 85.4 & 84.6 & 84.3 & 85.7 \\
\bottomrule
\end{tabular}
}
\end{table}

\subsection{Ablation Study}
To verify the effectiveness of our methods, we performed ablation experiments, with results shown in Table \ref{tab3}. Validation on both static and neuromorphic datasets fully confirms the superiority of AR-LIF and the validity of its dynamic components. Specifically, adaptive reset contributes positively, and the threshold adjustment strategy also plays a positive role. Comparative results indicate that threshold adjustment is more critical than the learnability of $\alpha^l$.

\begin{table}
\caption{The energy consumption for different tasks with the whole testing datasets.\label{tab4}}
\centering
\resizebox{0.48\textwidth}{!}{
\begin{tabular}{cccccccc}
\toprule
Dataset & Net & T & Params(M) & FLOPs(M) & MACs(M) & ACs(M) & Energy($\mu$J)\\
\midrule
CIFAR-10 & ResNet-18 & 4 & 11.2 & 1211.51 & 1.77 & 95.65 & 94.22\\
\midrule
CIFAR-100 & ResNet-18 & 4 & 11.2 & 1211.60 & 1.77 & 96.72 & 95.19\\
\midrule
Tiny-ImageNet & VGGSNN & 4 & 10.9 & 4847.11 & 763.08 & 209.56 & 3698.80\\
\midrule
CIFAR10-DVS & VGGSNN & 4 & 9.3 & 2724.21 & 428.09 & 57.35 & 2020.84\\
\midrule
DVS-Gesture & VGG-11 & 20 & 9.5 & 4872.13 & 1078.11 & 23.71 & 4980.63\\
\bottomrule
\end{tabular}
}
\end{table}

\begin{figure}
\centering
\includegraphics[width=2.8in, keepaspectratio]{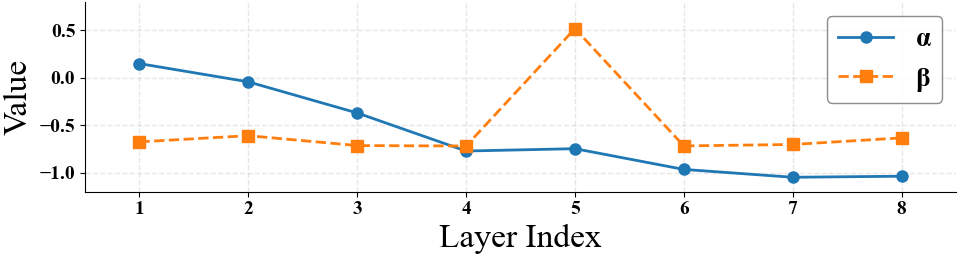}\\
\caption{Tracking of learnable parameters in each layer.
}
\label{fig_3}
\end{figure}

\subsection{Analysis of Computational Efficiency}
To comprehensively evaluate AR-LIF's performance, we conducted a computational cost analysis with results in Table \ref{tab4}. The energy consumption calculation therein follows \cite{yao2023sdt}. For ease of comparison, experiments on the Gesture dataset were performed with the time step set to 4. As seen in Fig. \ref{fig_2}, AR-LIF alters temporal dynamics, causing notable changes in each layer's spike firing rate. Overall, however, it shows a lower average spike firing rate, indicating a slight energy consumption advantage alongside excellent performance.

\subsection{Learnable Parameters Tracking}
As shown in Equation (\ref{eq9}), our method contains two learnable parameters. We visualized the learnable parameters of each layer. The experiment was conducted on the Gesture dataset with T=20, using the VGG-11 architecture. $\alpha^l$ is initialized to 1 and $\beta^l$ to 0. As shown in Fig. \ref{fig_3}, both $\alpha^l$ and $\beta^l$ have undergone sufficient learning and adjustment, confirming the effectiveness of parameter learnability.

\section{Conclusion}
\label{sec:conclusion}
In this work, we analyze the drawbacks of existing reset modes, and propose a spiking neuron model with adaptive reset, AR-LIF. Adaptive reset leverages accumulated input current and output spikes to dynamically and heterogeneously adjust the reset voltage. Combined with a spatiotemporally independent threshold adjustment strategy, AR-LIF performs excellently on static and neuromorphic datasets while maintaining an energy consumption advantage. Extensive ablation studies further confirm the effectiveness of our method.

\vfill\pagebreak

% References should be produced using the bibtex program from suitable
% BiBTeX files (here: strings, refs, manuals). The IEEEbib.bst bibliography
% style file from IEEE produces unsorted bibliography list.
% -------------------------------------------------------------------------
\bibliographystyle{IEEEbib}
\bibliography{refs}

\end{document}